\documentclass[letterpaper,11pt]{article}
\usepackage[utf8]{inputenc}
\usepackage[margin=1in]{geometry}

\usepackage[colorlinks,linkcolor=blue,citecolor=blue]{hyperref}

\usepackage{newfloat,framed}
\usepackage[ruled,vlined]{algorithm2e}
\usepackage[numbers]{natbib}
\usepackage{enumitem}
\usepackage{xspace}

\usepackage{amsmath,amsfonts,amssymb}
\usepackage{mathtools}
\usepackage{amsthm} 

\usepackage{mathabx}
\usepackage{relsize}

\usepackage[capitalise]{cleveref}
\usepackage{xcolor}
\usepackage{dsfont}
\usepackage{bm}

\usepackage{sectsty}
\allsectionsfont{\boldmath}

\usepackage{newfloat,framed}

\DeclareFloatingEnvironment[
    fileext=loa,
    listname=List of Algorithms,
    name=Algorithm,
    placement=h,
]{myalgorithm}
\crefname{myalgorithm}{Algorithm}{Algorithms} 

\newcommand{\wrapalgo}[2][0.9\linewidth]
{%
\begin{center}\setlength{\fboxsep}{5pt}\fbox{\begin{minipage}{#1}
#2
\end{minipage}}\end{center}
\vspace{-0.25cm}
}

\newcommand{\ignore}[1]{}

\theoremstyle{plain}
\newtheorem{theorem}{Theorem}
\newtheorem{lemma}[theorem]{Lemma}

\newtheorem*{theorem*}{Theorem}
\newtheorem*{lemma*}{Lemma}
\newtheorem*{corollary*}{Corollary}
\newtheorem*{proposition*}{Proposition}
\newtheorem*{claim*}{Claim}
\newtheorem*{fact*}{Fact}
\newtheorem*{observation*}{Observation}

\theoremstyle{definition}

\newtheorem*{definition*}{Definition}
\newtheorem*{remark*}{Remark}
\newtheorem*{example*}{Example}

 \theoremstyle{plain}
\newtheorem*{theoremaux}{\theoremauxref}
\gdef\theoremauxref{1}

 \newcommand{\secref}[1]{Section~\ref{#1}}

 \newcommand{\thmref}[1]{Theorem~\ref{#1}}

\DeclareMathAlphabet{\mathbfsf}{\encodingdefault}{\sfdefault}{bx}{n}

\let\Pr\relax
\DeclareMathOperator{\Pr}{\mathbb{P}}

\newcommand{\mycases}[4]{{
\left\{
\begin{array}{ll}
    {#1} & {\;\text{#2}} \\[1ex]
    {#3} & {\;\text{#4}}
\end{array}
\right. }}

\newcommand{\lr}[1]{\left(#1\right)}
\newcommand{\lrbig}[1]{\big(#1\big)}

\newcommand{\lrBigg}[1]{\Bigg(#1\Bigg)}
\newcommand{\lrbra}[1]{\left[#1\right]}

\newcommand{\lrset}[1]{\left\{#1\right\}}

\newcommand{\set}[1]{\{#1\}}
\newcommand{\abs}[1]{|#1|}

\newcommand{\distance}{\Delta}
\newcommand{\tdistance}{\Delta_{\mathcal{T}}}

\newcommand{\wt}[1]{\smash{\widetilde{#1}}}

\renewcommand{\O}{O}

\newcommand{\tO}{\wt{\O}}
\newcommand{\tTheta}{\wt{\Theta}}
\newcommand{\E}{\mathbb{E}}
\newcommand{\EE}[2][]{\E_{#1}\mkern-4mu\lrbra{#2}}

\newcommand{\ind}[1]{\mathds{1}\mkern-2mu\set{#1}}

\newcommand{\st}{\star}

\newcommand{\reals}{\mathbb{R}}

\newcommand{\sig}{\sigma}
\newcommand{\del}{\delta}

\newcommand{\TT}{\overline{T}}

\let\oldtfrac\tfrac
\renewcommand{\tfrac}[2]{\smash{\oldtfrac{#1}{#2}}}

\let\nablaold\nabla
\renewcommand{\nabla}{\nablaold\mkern-2mu}

\newcommand{\bb}{\mathbf{b}}
\newcommand{\bbb}{\overline{\bb}}
\newcommand{\taumax}{{\overline{\tau}}}

\newcommand{\val}{v}

\newcommand{\regret}{\textrm{Regret}}

\newcommand{\mregret}{\textrm{Regret}_\mathsf{MC}}

\newcommand{\buy}[2]{#1(#2)}

\newcommand{\p}{\rho}
\newcommand{\klm}{\textbf{EXP3MV}}

\def\expth{\ensuremath{\textsc{Exp3}}\xspace}
\def\klm{\ensuremath{\textsc{SMB}}\xspace}

\newcommand{\tell}{\wt{\ell}}

\newcommand{\bell}{\bm\bar{\ell}}
\newcommand{\cA}{\mathcal{A}}
\newcommand{\cT}{\mathcal{T}}

\newcommand{\LCA}{\mathsf{lca}}
\newcommand{\level}{\mathsf{level}}

\title{Bandits with Movement Costs and Adaptive Pricing}

\author{%
Tomer Koren\\
Google\\
\texttt{tkoren@google.com}
\and Roi Livni\\
Princeton University\\
\texttt{rlivni@cs.princeton.edu}\\
\and Yishay Mansour\\
Tel Aviv University\\
\texttt{mansour@tau.ac.il}%
}

\begin{document}

\maketitle
\thispagestyle{empty}

\begin{abstract}
We extend the model of Multi-armed Bandit with unit switching cost to
incorporate a metric between the actions. 
We consider the case where the metric
over the actions can be modeled by a complete binary tree, and the distance between
two leaves is the size of the subtree of their least common
ancestor, which abstracts the case that the actions are points on the
continuous interval $[0,1]$ and the switching cost is their
distance. 
In this setting, we give a new algorithm that establishes a regret of $\tO(\sqrt{kT} + T/k)$, where $k$ is the number of actions and $T$ is the time horizon.
When the set of actions corresponds to whole $[0,1]$ interval we can exploit our method for the task of bandit learning with Lipschitz loss functions, where our algorithm achieves an optimal regret rate of $\tTheta(T^{2/3})$, which is the same rate one obtains when there is no penalty for movements.

As our main application, we use our new algorithm to solve an adaptive pricing problem.
Specifically, we consider the case of a single seller faced with a stream of patient buyers. Each buyer has a private value and a window of time
in which they are interested in buying, and they buy at the lowest price in the window, if it is below their value.
We show that with an appropriate discretization of the prices, the seller can achieve a regret of $\tO(T^{2/3})$ compared to the best fixed price in hindsight, which outperform the previous regret bound of $\tO(T^{3/4})$ for the problem.
\end{abstract}

\clearpage
\setcounter{page}{1}


\section{Introduction}

Multi-Armed Bandit (MAB) is a well studied model in computational
learning theory and operations research. In MAB a learner repeatedly
selects actions and observes their rewards. The goal of the learner
is to minimize the regret, which is the difference between her loss
and the loss of the best action in hindsight. This simple model
already abstracts beautifully the {\em exploration-exploitation}
tradeoff, and allows for a systematic study of this important issue
in decision making. The basic results for MAB show that even when an
adversary selects the sequence of losses, the learner can guarantee
a regret of $\Theta(\sqrt{kT})$, where $k$ is the number of actions
and $T$ is the number of time steps
(\citealp{auer2002nonstochastic,audibert2009minimax}; see also
\citealp{BubeckC12}).

The simplicity of the MAB comes at a price.
Essentially, the system is stateless, and previous actions have no influence on the losses assigned to actions in the future.
A more involved model of sequential decision making is Markov Decision Processes
(MDPs) where the environment is modeled by a finite set of states, and
actions are not only associated with losses but also with stochastic transitions between states.
Unfortunately, for the adversarial setting there are mostly hardness
results even in limited cases~\citep{AbbasiBKSS13}.

Introducing switching costs is a step of incorporating dependencies
in the learner's action selection. The unit switching cost has a
unit cost per each changing of actions. In such a setting a tight
bound of $\tTheta(k^{1/3}T^{2/3})$ is known \citep{DekelDKP14}. Our
main goal is to extend this basic model to the case of MAB with
movement costs, where the cost associated with
switching between arms is given by a metric that determines the distance between any pair of arms. Such a model already introduces a
very interesting dependency in the action selection process for the
learner. Specifically, we study a metric between actions which is
modeled by a complete binary tree, where the distance between two
actions is proportional to the number of nodes in the subtree of
their least common ancestor. This abstracts the case where the arms are associated with $k$
points on the real line and the switching cost between arms is the absolute difference between the corresponding points (actually,
the tree metric only upper bounds distances on a line, but this
upper bound is sufficient for our applications). Note that we do not
assume that pairs of actions with low movement cost have similar losses: our model retains the full generality of the loss functions, and only imposes a metric structure on the cost of movement between arms.

Our main result is an efficient MAB algorithm, called the Slowly Moving Bandit (\klm) algorithm, that guarantees expected regret of at most $\tO(\sqrt{kT}+T/k)$. As we elaborate
later, this result implies that for $k \le T^{1/3}$ we can
achieve an optimal regret $\tTheta(T^{2/3})$, and for $k\ge T^{1/3}$
we obtain an optimal regret rate of $\tTheta(\sqrt{kT})$.
It is worth discussing the implication of our bound. The bound of $\tTheta(T^{2/3})$ for $k \le T^{1/3}$ is tight due to the lower bound
of \citet{DekelDKP14}, which applies already for $k=2$ actions. The bound of $\tTheta(\sqrt{kT})$ for $k \ge T^{1/3}$ is tight due to the classic lower bound for MAB even without movement costs \citep{auer2002nonstochastic}. Surprising, for a large action set (i.e., $k \ge T^{1/3}$) we lose nothing in the regret by introducing movement costs to the problem!
Another surprising consequence of our bound is that there is no loss in the regret by increasing the number of actions from $k=2$ to $k = \Theta(T^{1/3})$ when movement costs are present.


The main application of our \klm algorithm is for adaptive pricing with patient buyers~\citep{FeldmanKLMZ16}. In this adaptive pricing problem, we have
a seller which would like to maximize his revenue. He is faced with
a stream of patient buyers. Each buyer has a private value and a
window of time in which she would like to purchase the item. The
buyer buys at the lowest price in its window, in case it is below
its value. (The seller publishes sufficient prices into the future,
such that the buyer can observe all the relevant prices.)
The adaptive price setting is related to the MAB problem with
movement costs in the following way. The prices are continuous (say,
$[0,1]$) and the reward is the revenue gain by the seller. The
rewards are given by a one--sided Lipschitz function (specifically,
we receive the reward whenever we post a price which is at most the
private value, and zero otherwise). This allows us to apply our
bandit algorithm via discretization of the continuous space. The
challenge, though, remains to control the cost the seller pays which
stems from the buyer's patience.

The seller benchmark is the best single price. Using a single price
implies that the buyers either buy immediately, or never buy. The
movement cost models the loss due to having the buyer patient, which
can be thought as the difference between the price of the item when
the buyer arrives and the price at which it buys. (Note that there
might be a gain, since it might be that when the buyer arrives the
price is too high, but later lower prices make him buy. We ignore
this effect for now.) Our main result is that the seller can use our
\klm algorithm and guarantee a regret of at most $\tO(T^{2/3})$,
using $T^{1/3}$ equally-spaced prices. This is in contrast to a
regret of $\tO(T^{3/4})$ which is achieved by applying a standard switching cost technique together with a discretization argument \citep{FeldmanKLMZ16}.

It is interesting to observe qualitatively how our algorithm
performs. It is much more likely to make small changes than large
ones; roughly speaking, the probability of a change drops
exponentially in the magnitude of the change. Conceptually, this is
a highly desirable property of a pricing algorithm, and arguably, of
any regret minimization algorithm: we would like to slightly perturb
the prices over time without a sever impact on the buyers, and only
rarely make very large changes in the pricing.

Finally, another application of our algorithm is for the case that
we have continuous actions on an interval, and the losses of the
actions are Lipschitz. Our algorithm can handle movement cost which
are also Lipschitz on the interval. (We stress that in our
application the losses are deterministic and not stochastic.)

\subsection{Related Work}

With a uniform unit switching cost (i.e., when switching between any two actions has
a unit cost), it is known that there is a tight $\wt{\Omega}(k^{1/3}
T^{2/3})$ lower bound for the MAB problem \citep{DekelDKP14}, which is in contrast to the $\O(\sqrt{kT})$ regret upper bound without  switching costs.

Classical MAB algorithms such as \expth
\citep{auer2002nonstochastic} guarantee a regret of $\tO(\sqrt{kT})$
without movement costs. However, they are not guaranteed to move
slowly between actions, and in fact, it is known that \expth might make $\wt{\Omega}(T)$ switches between actions in the worst
case (see \citealp{DekelDKP14}), which makes
it inappropriate to directly handle movement costs.

Our adaptive pricing application follows the model of
\citet{FeldmanKLMZ16}. There, for a finite set of $k$ prices show a
matching bound of $\tTheta(T^{2/3})$ on the regret. For continuous
prices they remark that their upper bound can be used to derive an
$\tO(T^{3/4})$ regret bound. Our SMB algorithm improves this
regret bound to $\tO(T^{2/3})$. There is a slight difference in the
exact feedback model between \cite{FeldmanKLMZ16} and here: in
both models when a buyer arrives, the sell time is uniquely
determined; however, in \cite{FeldmanKLMZ16} the seller observes the
purchase only at the actual time of the sell, whereas here we assume the
seller observes the sell when the buyer arrives and decides when to
purchase. We remark, though, that as discussed in \cite{FeldmanKLMZ16}
all lower bounds  derived there apply to the current feedback model too.

There is a vast literature on online pricing
(e.g., \citealp{balcan2006approximation,balcan2008item,balcan2010sequential,bansal2010dynamic,besbes2009dynamic}).
The main difference of our adaptive pricing model is the patience of our buyers, which correlates between the prices at nearby time steps.

For the case of continuous prices and a single seller, when one consider \emph{impatient} buyers, a simple
discretization argument can be used to achieve a regret of $\tO(T^{2/3})$, and there exists
a similar lower bound of $\Omega(T^{2/3})$ \citep{KleinbergL03}.
More generally, learning Lipschitz functions on a closed interval
has been studied by \citet{Kleinberg04}, where an optimal
$\tTheta(T^{2/3})$ regret bound is shown via discretization. Our
results show that even if one adds a movement cost (which is the
distance) to the problem, there is no change in the regret.

There are many works on continuous action MAB
\citep{Kleinberg04,cope2009regret,auer2007improved,bubeck2011x,yu2011unimodal}.
Most of the works relate the change in the payoff to the change in
the action in various ways. Specifically, there is an extensive
literature on the Lipschitz MAB problem and various variants thereof
\citep{kleinberg2008multi,slivkins2011multi,slivkins2013ranked,kleinberg2010sharp},
where the expectation of the reward of arms have a Lipschitz
property.
%
We differ from that line of work. Our assumption is about the switching cost (rather than the losses) being related to the distance between the actions.

The work of \citet{guha2009multi} discusses a stochastic MAB, in the spirit of the Gittins index, where there is
both a switching cost and a play cost, and gives a constant
approximation algorithm. We differ from that work both in the model,
their model is stochastic and our is adversarial, and in the result, their is
a multiplicative approximation and our is a regret.

Approximating an arbitrary metric using randomized trees (i.e., $k$-HST) has a long
history in the online algorithms literature, starting with the work of \citet{Bartal96}. The
main goal is to derive a simpler metric representation (using
randomized trees) that will both upper and lower bound the given
metric. In this work we need only an upper bound on the metric, and
therefore we can use a deterministic complete binary tree.


\section{Setup and Formal Statement of Results}

\subsection{Bandits with Movement Costs}
\label{sec:setup-bandits}

In this section we consider the Multi-Armed Bandit (MAB) problem
with movement costs. In this problem, that can be described as a
game between an online learner and an adversary continuing for $T$
rounds, where there is a set $K=\{1, \ldots, k\}$ of $k \ge 2$ arms
(or actions) that the learner can choose from. The set of arms is
equipped with a metric $\Delta(i,j) \in[0,1]$ that determines the
movement distance between any pair of arms $i,j \in K$.

First, before the game begins, the adversary fixes a sequence
$\ell_1,\ldots,\ell_T \in [0,1]^k$ of loss vectors assigning loss
values in $[0,1]$ to the arms.\footnote{Throughout, we assume that the adversary is
\emph{oblivious}, namely, that it cannot react to the learner's
actions.} Then, on each round $t=1,\ldots,T$, the learner picks an
arm $i_t \in K$, possibly at random, and suffer the associated loss
$\ell_t(i_t)$. In addition to incurring this loss, the learner also
pays a cost of $\distance(i_t,i_{t-1})$ that results from her movement
from arm $i_{t-1}$ to arm $i_t$. At the end of each round $t$, the
learner receives \emph{bandit feedback}: she gets to observe the
single number $\ell_t(i_t)$, and this number only. (The movement
cost is common knowledge.)

The goal of the learner, over the course of $T$ rounds of the game,
is to minimize her expected movement-regret, which is defined as the
difference between her (expected) total costs---including both the
losses she has incurred as well as her movement costs---and the
total costs of the best fixed action in hindsight (that incur no
movement costs, since it is the same action in all time steps);
namely, the \emph{movement regret} with respect to a sequence $\ell_{1:T}$ of loss vectors and the metric $\Delta$ equals
\begin{align*}
\mregret(\ell_{1:T},\Delta)
=
\EE{ \sum_{t=1}^T \ell_t(i_t) + \sum_{t=2}^T
\distance(i_t,i_{t-1})} - \min_{i^\st \in K} \sum_{t=1}^T \ell_t(i^\st)
~.
\end{align*}
Here, the expectation is taken with respect to the player's
randomization in choosing the actions $i_1,\ldots,i_T$.


\paragraph{MAB with a tree metric.}

Our focus in this paper is on a metric induced over the actions by a complete binary tree $\cT$ with $k$ leaves.
We consider the MAB setting where each action $i$ is associated with a leaf of the tree $\cT$.
(For simplicity, we assume that $k$ is a power of two.)

We number the levels of the tree $\cT$ from the leaves to
the root. Let $\level(v)$ be the level of node $v$ in $\cT$, where the level of the leaves is $0$.
Given two leaves $i$ and $j$, let $\LCA(i,j)$ be their least common ancestor in $\cT$.
Then, given actions
$i$ and $j$ let $d_\cT(i,j)$ be the level of their least
common ancestor in $\cT$, i.e., $d_\cT(i,j) = \level(\LCA(i,j))$.
The movement cost between $i$ and $j$ is then
%
\begin{align}\label{eq:Delta}
\tdistance(i,j) = \tfrac{1}{k} 2^{d_{\cT}(i,j)} \in [0,1]
~.
\end{align}
Our first main result bounds that movement cost with respect to the given metric:
\begin{theorem} \label{thm:MAB}
There exists an algorithm (see \cref{alg:alg1} in \cref{sec:mab})  that for any sequence of loss functions $\ell_1,\ldots,\ell_T$ guarantees that
\[
\mregret(\ell_{1:T},\tdistance) = \tO\left(\sqrt{kT} + \frac{T}{k}\right)
~.
\]
\end{theorem}

For $k\ge T^{1/3}$ the theorem gives an optimal regret bound of $\tO(\sqrt{kT})$.
For $k\le T^{1/3}$, we can extend a binary tree with $k$ leaves by turning each leaf into a node whose subtree is a balanced binary tree and we obtain a new tree with at most $2 T^{1/3}$ leaves.
We then associate with each new leaf as its action the action induced by its parent at the level of original leaves. One can show that the movements between the level of the original actions is then controlled by $O(T^{2/3})$ and we can then exploit this construction to achieve a regret bound of $\tO(T^{2/3})$. In any movement cost problem with at least two arms of fixed constant distance, a lower bound regret of $2$-arm switching cost applies, hence we observe that these rates are optimal for every $k\le T$ \citep{DekelDKP14}.

\paragraph{Continuum-armed bandit with movement cost.}

We can apply \cref{alg:alg1} to the problem of learning Lipschitz
functions over the real line with movement regret associated with standard metric over the interval. In this setting we assume an arbitrary sequence
of functions $f_1,\ldots, f_T: [0,1] \mapsto [0,1]$ where each function $f_t$ is $L$-Lipschitz.
i.e.,
\begin{align*}
|f_t(x)-f_t(y)| \le L |x-y|
\qquad\qquad
\forall ~ x,y \in [0,1]
~.
\end{align*}
Let $x_t$ be the action selected by the player at time $t$. The
objective is then to minimize the \emph{movement regret}, defined:
\begin{align*}
\mregret(f_{1:T},|\cdot|)
=
\EE{\sum_{t=1}^T f_t(x_t) +\sum_{t=1}^T |x_t-x_{t+1}|} - \min_{x\in [0,1]} \sum_{t=1}^T f_t(x)
~.
\end{align*}

One application of our algorithm is a regret bound for Lipschitz functions:
\begin{theorem}\label{thm:lipschitz}
There exists an algorithm (based on \cref{alg:alg1})
that for every sequence of $L$-Lipschitz loss functions $f_1,\ldots,f_T$, with $L\ge 1$, achieves: 
\begin{align*}
\mregret(f_{1:T},|\cdot|)  = \tO\lrbig{ L^{1/3}T^{2/3} }
~.
\end{align*}
\end{theorem}

We emphasize that even without movement costs, there is an $\wt{\Omega}(T^{2/3})$ lower bound in this setting~\citep{Kleinberg04};
hence, the regret bound of \cref{thm:lipschitz} is essentially optimal.

We also note that the result, in fact, holds for any metric $\distance$ that is $L$-Lipschitz (for exact statement see \thmref{thm:lipschitz2}).

\subsection{Adaptive Pricing}
\label{sec:setup-pricing}

We consider the following model of online learning, with respect to
a stream of patient buyers with patience at most $\taumax$. In our
setting the seller posts at time $t=1$ prices $\p_1,\ldots,
\p_{\taumax+1}$ for the next $\taumax$ days in advance. Then at each
time step $t$ the seller posts price for the $t+\taumax$ day
$\p_{t+\taumax}$ and receives as feedback her revenue for day $t$.
The revenue at time $t$ depends on buyer $\bb_t$ and the sequence of
prices $\p_t,\p_{t+1},\ldots, \p_{t+\taumax}$ in the manner described below.

Each buyer $\bb_t$, in our setting, is a mapping from a sequence of
prices to revenues, parameterized by her \emph{value} $\val_t$ and
her patience $\tau_t$. The buyer proceed by observing prices $\p_t,
\ldots, \p_{t+\tau_t}$, and purchases the item at the lowest price
among these prices, if it does
not exceed her value. Thus the revenue from the buyer at time $t$ is
described as follows:
\[\buy{\bb_t}{\p_{t},\ldots,\p_{t+\taumax}}= \begin{cases}
\min\{\p_t,\ldots, \p_{t+\tau_t}\} & \mbox{if } \min\{\p_t,\ldots, \p_{t+\tau_t}\}\le \val_t ,\\
0 & \textrm{otherwise.}
\end{cases}\]
Note that at time $t$ the buyer decides whether it will purchase and
when. Here, we assume that the buyer also gets to order the good at day of arrival (at price and time decided by him according to his patience and private value), thus the seller observes the buyer's decision at time $t$, namely the feedback at time $t$ is given by
$\buy{\bb_t}{\p_t,\ldots,\p_{t+\taumax}}$. We note that this
 feedback model differs from \citet{FeldmanKLMZ16} where the buyer buy at day of purchase. However, we note  that both lower and upper bounds derived by Feldman et al. apply to our feedback model as noted there in the discussion.

Our objective is to construct an algorithm that minimizes the regret
which is the difference between revenue obtained by the best fixed
price in hindsight and the expected revenue obtained by the seller, given a sequence $\bb_{1:T}$ of buyers:
\[
\regret(\bb_{1:T})
= \max_{\p^*\in P} \sum_{t=1}^T \buy{\bb_t}{\p^*,\ldots ,\p^*}
- \EE{\sum_{t=1}^T \buy{\bb_t}{\p_{t},\ldots \p_{t+\taumax}}}
.
\]

Our main result with respect to adaptive pricing is as follows:
\begin{theorem}\label{thm:pp}
There exists an algorithm (see \cref{alg:pp} in \cref{sec:pricing}) that for any sequence of buyers $\bb_1,\ldots,\bb_T$ with maximum patience $\taumax$ achieves the following regret bound:
\[\regret(\bb_{1:T}) = \tO(\taumax^{1/3}T^{2/3})~.\]
\end{theorem}


It is interesting to note that even though a lower bound of
$\Omega(T^{2/3})$ stems from two different sources we can still
achieve a regret rate of $\tO(T^{2/3})$. Indeed,
\citet{KleinbergL03} showed that optimizing over the continuum
$[0,1]$ leads to a lower bound of $\Omega(T^{2/3})$, irrespective of
the patience of the buyers. Second, \citet{FeldmanKLMZ16} showed
that whenever the seller wishes to optimize between more than two
prices, a lower bound of $\Omega(T^{2/3})$ holds for patient buyers.

In this work we deal with both obstacles together---patient buyers and optimization over the $[0,1]$ interval---yet the two obstacles can be dealt without leading to a regret bound that is necessarily worse then each obstacle alone.

Our solution to the adaptive pricing problem is based on employing a MAB
with movement costs algorithm that allows small change in the prices.
The reason one needs to employ an algorithm with small movement cost
stems from the memory of the buyers: roughly speaking, whenever the
seller encounters a buyer with patience,  the potential revenue of
the seller will be the revenue at time $t$ minus any discount price
that buyer may encounter on future days. Indeed, for the case of two
prices, \citet{FeldmanKLMZ16} constructed a sequence
of buyers that reduces the problem to MAB with switching cost: a
step in demonstrating a $\Omega(T^{2/3})$ regret bound: thus a
fluctuation in prices is indeed a cause for a high regret.


\section{Overview of the approach and techniques}
\label{sec:techniques}

In this section we give an informal overview of the main ideas in the paper and describe the techniques used in our solution.
We begin with the main ideas behind our main result: an optimal and
efficient algorithm for MAB problems with movement costs.
Later we continue with the adaptive pricing problem, and show how it
is abstracted as an instance of the MAB problem with movement costs.

\paragraph{From continuum-armed to multi-armed.}

In our main applications, we consider actions that are associated to points on the interval $[0,1]$ equipped with the natural metric $\Delta(x,y) = \abs{x-y}$.
As a preliminary step, we use discretization in order to make the action space finite and capture the setting by the MAB framework.
That is, we reduce the problem of minimizing regret over the entire $[0,1]$ interval to regret minimization over $k$ actions associated with the equally-spaced points $K = \set{\frac{1}{k},\frac{2}{k},\ldots,1}$.
Our challenge is to then to design a regret minimization algorithm over $\cA$ whose cumulative movement cost with respect to the metric $\wt{\Delta}(i,j) = \abs{i-j}/k$ is bounded.

Our approach builds upon the basic techniques underlying
the \expth algorithm for the basic MAB problem, which we recall here. 
\expth maintains over rounds a distribution $p_t$ over the $k$ actions and
chooses an action $i_t \sim p_t$; thereafter, it updates its
sampling distribution multiplicatively via  $p_{t+1}(i) \;\propto\;
p_t(i) \cdot \exp(-\eta \bell_t(i))$, where $\bell_t$ is an unbiased
estimator of true loss vector $\ell_t$ constructed using only the
observed feedback $\ell_t(i_t)$. Specifically, the estimator used by
\expth is
\begin{align*}
\bell_t(i) &= \frac{\ind{i_t = i}}{p_t(i)} \ell_t(i_t) \qquad\quad
\forall ~ i \in K ~.
\end{align*}
A simple computation shows that $\bell_t$ is indeed an unbiased
estimator of $\ell_t$, namely that $\E[\bell_t] = \ell_t$, and the
crucial bound for $\expth$ is then obtained by controlling a
variance term of the form $\E[p_t \cdot \bell_t^2]$, and showing
that it is of the order $\tO(k)$ at all rounds $t$. This in turn
implies the $\tO(\sqrt{kT})$ bound of \expth.

\paragraph{Controlling movements with a tree.}

As a first step in controlling the movement costs of our algorithm, 
one can think of an easier problem of controlling the number of times the algorithm switches between actions in the left part of the interval, namely in $A_L = \set{\frac{1}{k},\ldots,\frac{1}{2}}$, and actions in the right part of the interval, $A_R = \set{\frac{1}{2}+\frac{1}{k},\ldots,1}$.
Indeed, since each such switch might incur a high movement cost (potentially close to $1$), any algorithm for MAB with movement costs must avoid making such switches too often.
In principle, a solution to this simpler problem can be then lifted to a solution to the actual movement costs problem by applying it recursively to each side of the interval.

The thought experiment above motivates our tree-based metric: this metric assigns a fixed cost of $1$ to any movement between the left and right parts of the interval---that correspond to the topmost left and right subtrees---and recursively, a cost of $2^d/k$ for any movement between subtrees in level $d$ of the tree.
The tree metric is always an upper bound on the natural metric on the interval, namely $\wt{\Delta}(i,j) \le \frac{1}{k}2^{d_{\mathcal{T}}(i,j)} = \wt\Delta_{\mathcal{T}}(i,j)$, so that controlling movement costs with respect to $\wt\Delta_{\mathcal{T}}$ suffices for controlling movement costs with respect to the natural distance on $[0,1]$.
While this upper bound might occasionally be very loose,%
\footnote{For example, the distance between $\frac{1}{2}-\frac{1}{k}$ and $\frac{1}{2}+\frac{1}{k}$ according to the metric $\Delta_\cT$ is $1$.}
the tree-metric effectively captures the difficulties of the original movement costs problem with the natural metric over $[0,1]$.

Hence, we can henceforth focus on constructing an algorithm with low
movement costs with respect to a tree-based metric over a full binary tree.
To accomplish this, we will regulate the probability of switching the ancestral node. 
Namely, if we denote by $A_d(i)$ the subtree at level $d$ of the tree
containing action $i$, our goal is to design an algorithm that switches between actions $i$ and $j$ such that $A_d(i) \ne A_d(j)$ with probability at most $2^{-d}$.
%
This would ensure that the expected contribution of level $d$ in the tree to the movement cost of the algorithm is $O(1/k)$ per round. Indeed, switching between subtrees at level $d$ (while not making a switch at higher levels) results with a movement cost of roughly $2^d/k$.
Overall, the contribution of all layers in the tree to the total movement cost would then be $O((T/k)\log{k})$, as required.


\paragraph{Lazy sampling.}

Our challenge now is to construct an algorithm that switches infrequently between subtrees at higher levels of the tree.
However, recall that typical bandit algorithms choose their actions $i_1,\ldots,i_T$ at random from sampling distributions $p_1,\ldots,p_T$ maintained throughout the evolution of game.
In order to guarantee that consecutive actions $i_t$ and $i_{t-1}$ will belong to the same subtree with high probability, the algorithm would have to sample $i_{i}$ in a way which is highly correlated with the preceding action $i_{t-1}$.

Suppose that the marginals of the subtrees at some level $d$ does not change between the distributions $p_{t-1}$ and $p_{t}$; namely, that the cumulative probability assigned to the leaves of each such subtree by both $p_{t-1}$ and $p_t$ is the same.
In this case, we argue that we can sample our new action~$i_{t}$ at time $t$, based on the preceding action $i_{t-1}$, from the conditional distribution $p_t(\cdot \mid A_{d}(i_{t-1}))$.
In other words, if we think of sampling an action $i$ from $p_t$ as sampling a path in the tree leading to the leaf associated with $i$, then for determining $i_t$ on round $t$ we copy the top $d$ edges from the path at time $t-1$, and only sample the remaining bottom edges (those contained in the subtree $A_d(i_{t-1})$) according to the new distribution $p_t$.
Intuitively, this can be justified because the distribution of the top $d$ edges in the path leading to $i_{t}$ is the same as that of the top $d$ edges in the path leading to $i_{t-1}$, 
so we may as well keep the random bits associated with them and only resample bits associated with the remaining edges from fresh.






The lazy sampling scheme sketched above raises a major difficulty in the analysis: since $i_t$ is sampled from a conditional of $p_t$ that might be very different from $p_t$ itself, it is no longer clear that $i_t$ is distributed according to the ``correct'' distribution. 
In other words, conditioned on $p_t$ (which intuitively is a summary of the past), the random variable $i_t$ is certainly not distributed according to $p_t$.
Nevertheless, our analysis demonstrates a crucial property of the distributions $p_t$ maintained the sampling scheme, which is sufficient for the regret analysis: we show that for all subtrees $A$ at all levels of the tree, it holds that
$$
\EE{\frac{\ind{i\in A}}{p_t(A)}}=1 ~.
$$
That is, even though $i_t$ is sampled indirectly from $p_t$, it is still distributed according to $p_t$ in a certain sense.

\paragraph{Rebalancing the marginals.}

The lazy sampling we described above reduced the problem of controlling the frequency of movements in the actions $i_1,\ldots,i_T$, to controlling the frequency in which the marginal distribution of $p_1,\ldots,p_T$ over subtrees is updated by our algorithm.
Next, we describe how the latter is accomplished (where the
frequency of update is exponentially-decreasing with the level of
the subtree).
To illustrate the technique, let us consider an easier problem:
instead of demanding infrequent updates for subtrees in all levels,
we shall only attempt to rebalance the marginals at the topmost
level, with the goal of making them being updated with probability
at most $2^{-D} = 1/k$.
We will demonstrate how the estimator $\tell_t$ can be modified in a way that induces such infrequent updates at the top level.
Denote the left subtree at the top level by $A_L$ (containing actions $\frac{1}{k},\ldots,\frac{1}{2}$) and the right topmost subtree by $A_R$ (containing actions $\frac{1}{2}+\frac{1}{k},\ldots,1$).
First, we choose
\begin{align*}
\sig_t
=
\mycases
    {1-\frac{1}{\del}}{with probability $\del$;}
    {1}{with probability $1-\del$.}
\end{align*}
Then, for $A \in \set{A_L,A_R}$ we set
\begin{align*}
\tell_t(i)
&=
\bell_t(i) - \frac{\sig_t}{\eta} \log\lr{ \sum_{j \in A} \frac{p_t(j)}{p_t(A)} e^{-\eta \bell_t(j)} }
\qquad\quad
\forall~i \in A
~.
\end{align*}
Here, $\bell_t$ is the basic \expth estimator discussed earlier.
In terms of estimation, $\tell_t$ is still an unbiased estimator of
the true vector $\ell_t$: since $\E[\sig_t] = 0$ it follows that
$\E[\tell_t] = \ell_t$. However, the added term has a balancing
effect at the top level of the tree: a simple computation reveals
that if $\sig_t = 1$ (which occurs with high probability), the
multiplicative update of the algorithm applied on the vector
$\tell_t$ ensures that $p_t(A_L) = p_{t+1}(A_L)$ and $p_t(A_R) =
p_{t+1}(A_R)$. In other words, with probability $1-\del$, the
cumulative (i.e., marginal) probability of both subtrees at the top
level is remained fixed between rounds $t$ and $t+1$.

The balancing effect we achieved comes at a price: for small values
of $\delta$ the magnitude of $\tell_t$ becomes large, as it might be
the case that $\sig_t \approx -1/\del$. Nevertheless, it is not hard
to show that the variance term $\E[ p_t \cdot \tell_t^2]$ is bounded
by $\O(k + 1/\del)$. In particular, for $\del = 1/k$ we retain a
variance bound of $O(k)$, while changing the marginals of the two
top subtrees with probability no larger than $1/k$. As a result, by
sampling accordingly from the slowly-changing distributions $p_t$ we
can ensure that the movements at the top level contribute at most
$O(T/k)$ to the total movement cost of the algorithm.




Evidently, the estimator described above only remedies the problem
at the top level, and the movement costs at lower levels of the tree
might still be very large (effectively, within each subtree the
algorithm does nothing but simulating \expth on the leaves).
%
Still, using a similar yet more involved technique we can induce a
balancing effect at all levels simultaneously
and ensure that the marginal probabilities of the
subtrees at level~$d$ are modified by the algorithm with probability
at most $2^{-d}$. The construction adds a balancing term
corresponding to each level of the tree in a recursive manner that
takes into account the balancing terms at lower levels. 

\paragraph{From adaptive pricing to bandits.}

We now discuss how to reduce adaptive pricing with patient buyers to
a MAB problem with movement costs.
We employ a reduction similar to the one used by
\cite{KleinbergL03}; however, the patience of the buyers introduce
some difficulties, as we discuss below. For now, we ignore the
buyers' patience and give the idea of the reduction in the simplest
case.

Intuitively, in order to adaptively pick prices from the interval
$[0,1]$ so as to minimize regret with respect to the best fixed
price in hindsight, we could directly apply a standard MAB
algorithm, e.g., \expth, over a discretization $\cA =
\set{\tfrac{1}{k},\tfrac{2}{k},\ldots,1}$ of the interval, treating
each of the $k$ prices as an arm that generates a reward 
whenever it is pulled. 
Furthermore, since the buyers' valuations are not disclosed after purchase, the feedback observed by the seller is very limited and nicely captured by the MAB abstraction.
Since the buyers' valuations are one-sided Lipschitz, the best price in $\cA$ will
lose at most $O(T/k)$ in total revenue as compared to the best fixed
price in the entire $[0,1]$ interval. Thus, provided an algorithm that achieves $\tO(\sqrt{kT})$ expected regret with respect to the best
price in $\cA$, we could pick $k= \Theta(T^{1/3})$ and obtain the optimal
$\tO(T^{2/3})$ regret for the pricing problem.

\paragraph{Patient buyers and movement costs.}

A main complication in the above MAB approach arises from the
buyers' patience: the revenue extracted from a single buyer is
determined not only by the price posted by the seller on the day of
the buyer's arrival, but also by prices posted on the subsequent
days subject to the buyer's patience. As a result, if the seller
change prices abruptly on consecutive days, a strategic buyer---that
purchases in the minimal price, if at all---could make use of this
fact to gain the item at a lower price, which lowers the revenue of
the seller. Roughly speaking, the latter additional cost to the
seller is controlled by
the absolute difference between the prices she posted at consecutive
days.
%
Thus, the pricing problem with patient buyers can be reduced to a
MAB problem with movement costs, where the online player suffers an
additional movement cost each time she changes actions, and the
movement cost is determined by the metric (absolute value distance)
between the respective actions.

The reduction sketched above is made precise in \cref{sec:pricing},
where we also address an additional difficulty stemming from the
adaptivity of the \emph{feedback signal} observed by the seller: the
latter is contaminated by the effect of prices posted at earlier
rounds on the buyers, and has to be treated carefully.


\section{The Slowly Moving Bandit Algorithm}
\label{sec:mab}

In this section we present the Slowly Moving Bandit (\klm) algorithm: our optimal algorithm for the Multi-armed bandit problem with movement costs.

In order to present the algorithm we require few additional notations.  Recall that in our setting, we consider a complete binary tree of depth $D =\log_2{k}$ whose leaves are identified with the actions $1,\ldots,k$
(in this order). For any level $0 \le d \le D$ and arm $i \in K$,
let $A_d(i)$ be the set of leaves that share a common ancestor with $i$ at level $d$ (where level $d=0$ are the singletons). We denote
by $\cA_d$ the collection of all $k/2^d$ subsets of leaves:
\begin{align*}
\cA_d = \lrset{ \set{1,\ldots,2^d}, \set{2^d+1,\ldots,2 \cdot 2^d}, \ldots, \set{ k-2^d+1, \ldots, k } }
\qquad\quad
\forall ~ 0 \le d \le D
~.
\end{align*}

The \klm algorithm is presented in \cref{alg:alg1}. The
algorithm is based on the multiplicative update method,
and in that sense is reminiscent of the \textsc{Exp3} algorithm
\citep{auer2002nonstochastic}. Similarly to \textsc{Exp3}, the
algorithm computes at each round $t$ an estimator $\tell_t$ to the
true, unrevealed loss vector $\ell_t$ using the single loss value
$\ell_t(i_t)$ observed on that round.

As discussed in \cref{sec:techniques}, in addition to being an (almost) unbiased estimate for the true loss vector, the estimator $\tell_t$ used by \klm has the
additional property of inducing slowly-changing sampling
distributions $p_t$, that allow for sampling the actions $i_t$ in a
way that the overall movement cost is controlled.
This is achieved by choosing at random, at each round $t$, a level $d_t$ of the tree to be rebalanced by the algorithm using the balancing vectors $\bell_{t,d}$.
For reasons that will become apparent later on, the level $d_t$ is determined by choosing a random sign $\sig_{t,d}$ for each level $d$ in the tree and identifying the bottommost level with a negative sign.
Then, as we show in the analysis, the terms $\bell_{t,d}$ defined using the signs $\sig_{t,d}$ have a balancing effect at levels $d \ge d_t$.

A major difficulty inherent to our approach, also common to many bandit optimization settings (e.g., \citealp{dani2007price,alon2015online,bubeck2016kernel}), is the fact that the estimated losses $\tell_t(i)$ might receive negative values that are very high in absolute value. 
Indeed, the balancing term $\bell_{t,d}$ corresponding to level $d$ is roughly as large as $2^d/p_t(i_t)$, and might appear in negative sign in $\tell_t$.
\cref{alg:alg1} resolves this issue by zeroing-out the estimator $\tell_t$ whenever it chooses an action whose probability is too small, which ensures that the $\bell_{t,d}$ terms never become too large.
We remark that the standard approaches used to resolve such issues (the simplest of which is mixing the distribution $p_t$ with the uniform distribution over the $k$ actions) fail in our case, as they break the rebalancing effect which is tailored to the specific multiplicative update of the algorithm.


\begin{myalgorithm}[ht]
\wrapalgo[0.67\textwidth]{
Initialize $p_1 = u$, $d_0 = D$ and $i_0 \sim p_1$; for $t=1,\ldots,T$:
\begin{enumerate}[nosep,label=(\arabic*)]
\item
Choose action $i_t \sim p_t(\,\cdot \mid A_{d_{t-1}}(i_{t-1}))$, observe loss $\ell_t(i_t)$
\item
Choose $\sig_{t,0},\ldots,\sig_{t,D-1} \in \set{-1,+1}$ uniformly at random;\\
let $d_t = \min\set{0 \le d \le D : \sig_{t,d} < 0}$ where $\sig_{t,D} = -1$
\item
Compute vectors $\bell_{t,0},\ldots,\bell_{t,D-1}$ recursively via
\begin{flalign*} 
\bell_{t,0}(i)
=
\frac{\ind{i_t = i}}{p_t(i)} \ell_t(i_t)
~,
&&
\end{flalign*}
and for all $d \ge 1$:
\begin{flalign*}
\bell_{t,d}(i)
=
-\frac{1}{\eta} \log\lr{ \sum_{j \in A_{d}(i)} \frac{p_t(j)}{p_t(A_{d}(i))} e^{ -\eta (1+\sig_{t,d-1}) \bell_{t,d-1}(j) } }
&&
\end{flalign*}
\item
Define $B_t = \set{\text{$p_t(A_d(i_t)) < 2^d\eta$ for some $0 \le d < D$}}$ and set
\begin{align*}
\tell_t
=
\mycases
	{0}{if $i_t \in B_t$;}
	{\bell_{t,0} + \sum_{d=0}^{D-1} \sig_{t,d} \bell_{t,d}}{otherwise}
\end{align*}
\item
Update:
\begin{align*}
p_{t+1}(i) = \frac{ p_t(i) \, e^{-\eta\tell_t(i)} }{ \sum_{j=1}^k
p_t(j) \, e^{-\eta \tell_t(j)} } \qquad \forall ~ i \in K
\end{align*}
\end{enumerate}
}
\caption{The \klm algorithm.} \label{alg:alg1}
\end{myalgorithm}
The following theorem is the main result of this section. \thmref{thm:MAB} is an immediate corollary.

\begin{theorem} \label{thm:main}
For any sequence of loss functions $\ell_1,\ldots, \ell_T$, The \klm algorithm (\cref{alg:alg1}) guarantees that
\[
\regret(\ell_{1:t}) = O\lr{ \frac{\log k}{\eta} + \eta T k\log{k} }.
\]
In particular, by setting $\eta=1/\sqrt{kT}$ the expected regret of the algorithm is bounded by $\O(\sqrt{Tk}\log{k})$.
Furthermore, for the metric $\tdistance$ (see \cref{eq:Delta}), the expected total movement cost of the algorithm is $\E[ \sum_{t=2}^T \tdistance(i_t,i_{t-1}) ] = O((T/k)\log k)$.
\end{theorem}

The rest of the section focuses on proving \cref{thm:main}.
We begin by stating a useful technical bound that we use throughout our analysis to control the magnitude of the balancing vectors $\bell_{t,d}$.
For a proof of the lemma, see \cref{sec:mabproofs} below.


\begin{lemma} \label{lem:bell}
For all $t$ and $0 \le d < D$ the following holds almost surely:
\begin{align} \label{eq:bell1}
0 \le \bell_{t,d}(i) \le \frac{\ind{i_t \in A_d(i)}}{p_t(A_d(i))}
\prod_{h=0}^{d-1} (1+\sig_{t,h}) \qquad \forall ~ i \in K \,.
\end{align}
In particular, if $\sig_{t,h} = -1$ then $\bell_{t,d} = 0$ for all $d > h$.
\end{lemma}

One useful implication of the lemma is that, since $\bell_{t,d} = 0$ for all $d > d_t$, we can express our estimator $\tell_t$ in the following equivalent form:
\begin{align} \label{eq:tell-equiv}
\tell_t
=
\bell_{t,0} - \bell_{t,d_t} + \sum_{h=0}^{d_t-1} \bell_{t,h}
~.
\end{align}

\subsection{Rebalancing the marginals}

Our first step is to show that the marginals of the distributions $p_t$ over subtrees of actions are not modified by the algorithm with high probability, as a result of adding the balancing vectors $\bell_{t,d}$.

\begin{lemma} \label{lem:movement}
For all $d \ge d_t$ we have that $p_{t+1}(A) = p_t(A)$ for all $A \in \cA_{d}$.
\end{lemma}

For the proof, we require the next technical result about the balancing vectors $\bell_{t,d}$ computed by the algorithm.

\begin{lemma} \label{lem:balance}
If $\sig_{t,0} = \ldots = \sig_{t,d-1} = 1$ then:
\begin{align*} 
\sum_{i \in A} p_t(i) e^{-\eta \bell_{t,d}(i)}
=
\sum_{i \in A} p_t(i) e^{-\eta \tell_{t,d}(i)}
\qquad\quad
\forall ~ A \in \cA_d
~,
\end{align*}
where $\tell_{t,d} = \bell_{t,0} + \sum_{h=0}^{d-1} \bell_{t,h}$.
\end{lemma}

\begin{proof}
The proof proceeds by induction on $d$. For the base case $d=0$, the claim follows trivially as $\bell_{t,0} = \tell_{t,0}$.
Next, we assume the claim is true for some value of $d \ge 0$ and prove it for $d+1$.
Pick any $A \in \cA_{d+1}$ and write $A = A_1 \cup A_2$ where $A_1,A_2$ are disjoint sets from $\cA_d$.
Notice that the vector $\bell_{t,d}$ is uniform over $A_1$ and $A_2$, namely $\bell_{t,d}(i) = c_{A_1}$ for all $i \in A_1$ for some $c_{A_1} \ge 0$, and similarly $\bell_{t,d}(i) = c_{A_2}$ for all $i \in A_2$ for some $c_{A_2} \ge 0$.
Hence, we have
\begin{align*}
\sum_{i \in A} p_t(i) e^{-\eta \tell_{t,d+1}(i)}
&=
\sum_{i \in A} p_t(i) e^{-\eta \tell_{t,d}(i)} e^{-\eta \bell_{t,d}(i)}
\\
&=
e^{-\eta c_{A_1}} \sum_{i \in A_1} p_t(i) e^{-\eta \tell_{t,d}(i)} + e^{-\eta c_{A_2}} \sum_{i \in A_2} p_t(i) e^{-\eta \tell_{t,d}(i)}
\\
&=
e^{-\eta c_{A_1}} \sum_{i \in A_1} p_t(i) e^{-\eta \bell_{t,d}(i)} + e^{-\eta c_{A_2}} \sum_{i \in A_2} p_t(i) e^{-\eta \bell_{t,d}(i)}
\\
&=
\sum_{i \in A} p_t(i) e^{-\eta \bell_{t,d}(i)} e^{-\eta \bell_{t,d}(i)}
\\
&=
\sum_{i \in A} p_t(i) e^{-2\eta \bell_{t,d}(i)}
~,
\end{align*}
where the third equality uses the induction hypothesis.
On the other hand, by the recursive definition of $\bell_{t,d+1}$ and the fact that $\bell_{t,d+1}$ is uniform over $A$, we have
\begin{align*}
\sum_{i \in A} p_t(i) e^{-\eta \bell_{t,d+1}(i)}
=
p_t(A) \sum_{i \in A} \frac{p_t(i)}{p_t(A)} e^{-\eta (1+\sig_{t,d}) \bell_{t,d}(i)}
=
\sum_{i \in A} p_t(i) e^{-2\eta \bell_{t,d}(i)}
~.
\end{align*}
Combining both observations, we obtain
\begin{align*}
\sum_{i \in A} p_t(i) e^{-\eta \bell_{t,d+1}(i)}
=
\sum_{i \in A} p_t(i) e^{-\eta \tell_{t,d+1}(i)}
\end{align*}
which concludes the inductive argument.
\end{proof}

We can now prove \cref{lem:movement}.

\begin{proof}[Proof of \cref{lem:movement}]
It is enough to prove that $p_{t+1}(A) = p_t(A)$ for all $A \in \cA_{d_t}$, as each set in $\cA_{d}$ for $d>d_t$ is a disjoint union of sets from $\cA_{d_t}$.

Observe that if $i_t \in B_t$ (see \cref{alg:alg1} for the definition of $B_t$) then $\tell_t = 0$ and the claim is certainly true as $p_{t+1} = p_t$ in this case. 
Thus, we henceforth assume that $i_t \notin B_t$, in
which case $\tell_t = \tell_{t,d_d} - \bell_{t,d_d}$ where $\tell_{t,d_t}
= \bell_{t,0} + \sum_{h=0}^{d_t-1} \bell_{t,h}$ (recall \cref{eq:tell-equiv}).
Now, pick any $A \in \cA_{d_t}$ and
$j \in A$. Since $\bell_{t,d_t}(i) = c_A$ for all $i \in A$ for some
$c_A \ge 0$, and using  \cref{lem:balance} we obtain
\begin{align} \label{eq:induc}
e^{-\eta c_A}
=
\sum_{i \in A} \frac{p_t(i)}{p_t(A)} e^{-\eta \bell_{t,d_t}(i)}
=
\sum_{i \in A} \frac{p_t(i)}{p_t(A)} e^{-\eta \tell_{t,d_t}(i)}
~.
\end{align}
%
On the other hand, from $\tell_t = \tell_{t,d_t} - \bell_{t,d_t}$ it follows that $e^{-\eta \tell_t(i)} = e^{-\eta \tell_{t,d_t}(i)} / e^{-\eta c_A}$ for all $i \in A$, and by \cref{eq:induc} we have
\begin{align*}
\sum_{i \in A} p_t(i) e^{-\eta \tell_t(i)}
=
\frac{\sum_{i \in A} p_t(i) e^{-\eta \tell_{t,d_t}(i)}}{e^{-\eta c_A}}
=
p_t(A)
~.
\end{align*}
In words, the multiplicative update does not change the probabilities of the sets in $\cA_{d_t}$, hence $p_{t+1}(A) = p_t(A)$ for all $A \in \cA_{d_t}$ as required.
\end{proof}

\subsection{Lazy sampling}

Our next step is to show that the sampling scheme employed by \cref{alg:alg1} is valid and gives rise to low movement costs on expectation.
Specifically, we would like to show that in a certain sense, the action $i_t$ on round $t$ is distributed in expectation according to the distribution $p_t$, even though it is sampled from a conditional of $p_t$ in a way that is highly correlated with the preceding action $i_{t-1}$.
Furthermore, we will show that the correlations in the sampling scheme are designed in a way that the expected movement between consecutive actions is small.
These properties are formalized in the following lemma.

\begin{lemma} \label{lem:sampling}
For all $t$ and $0 \le d < D$ the following hold:
\begin{itemize}
\item for all $A \in \cA_d$ we have
\begin{align} \label{eq:property}
\EE{\frac{\ind{i_t\in A}}{p_t(A)}}=1 ~;
\end{align}
\item with probability at least $1-2^{-(d+1)}$, we have that $A_d(i_t)=A_d(i_{t-1})$.
\end{itemize}
\end{lemma}

\cref{eq:property} is central to our analysis below, and virtually all of our probabilistic arguments involving the random variables $i_t$ and $p_t$ will be based on this property.
We remark that if we were to sample $i_t$ directly from the distribution specified by $p_t$, then \cref{eq:property} would have been trivially true. However, the $i_t$ are sampled from a conditional of $p_t$ that might be very different from $p_t$ itself; nevertheless, the lemma shows that \cref{eq:property} still continues to hold under the skewed sampling process.

\cref{lem:sampling} also implies the slow-movement property of the algorithm: at the high levels of the tree, where the subtrees are ``wide'', the actions $i_t$ and $i_{t-1}$ are very likely to belong to the same subtree.
The probability of switching subtrees increases exponentially with the level in the tree: at the lower levels, where the subtrees are ``narrow'', subtree switches may occur more often as the movement cost incurred by such switches is low.


\begin{proof}[Proof of \cref{lem:sampling}]
The second statement is true since we pick $i_{t+1} \sim p_t(i \mid A_{d_{t}}(i_{t}))$, so that $A_d(i_{t+1}) \ne A_d(i_{t})$ can occur only if $d<d_{t}$.
This happens
with probability $2^{-(d+1)}$.

Next, we show \cref{eq:property} by induction on $t$.
For $t=1$ the statement is true since $i_1 \sim p_1$.
For the induction step, condition on $d_t$ and fix any $d \ge d_t$ and $A \in \cA_d$. By \cref{lem:movement} we have that $p_t(A)=p_{t+1}(A)$. Also $i_t\in A$ if and only if $i_{t+1}\in A$, since $d \ge d_t$ implies that $i_t\in A$ if and only if $A_{d_t}(i_t)\subseteq A$ and $A_{d_t}(i_{t+1})=A_{d_t}(i_{t})$.
Hence, we have
\begin{align} \label{eq:larged}
\EE{\frac{\ind{i_{t+1}\in A}}{p_{t+1}(A)} \;\middle|\;  d_t}
=
\EE{\frac{\ind{i_t\in A}}{p_t(A)}\;\middle|\;  d_t}
=
\EE{\frac{\ind{i_t\in A}}{p_t(A)}}
=
1
~,
\end{align}
where the last equality holds true since $d_t$ depends solely on $\sig_{t,0},\ldots,\sig_{t,D-1}$ which are independent of $i_t$ and $p_t$ (note that this equality then holds for \emph{any} set $A$, regardless of the fact that $A \in \cA_d$).

Next, we consider any $d < d_t$ and $A \in \cA_d$. Let $A' \in \mathcal{A}_{d_t}$ be the subtree such that $A\subseteq A'$, and recall that $i_{t+1}\sim p_{t+1}(i \mid A_{d_t}(i_t))$. Hence,
\begin{align} \label{eq:in}
\EE{ \frac{\ind{i_{t+1}\in A}}{p_{t+1}(A)} \;\middle|\; i_t \in A',p_{t+1},d_t }
=
\EE{\frac{\ind{i_{t+1} \in A}}{p_{t+1}(A \mid A')p_{t+1}(A')} \;\middle|\; i_t \in A',p_{t+1},d_t }
=
\frac{1}{p_{t+1}(A')}~.
\end{align}
Since $i_t \in A'$ implies that $i_{t+1}\in A'$, we have
\begin{align} \label{eq:out}
\EE{\frac{\ind{i_{t+1}\in A}}{p_{t+1}(A)}\;\middle|\; d_t,p_{t+1}}
=
\EE{ \ind{i_{t+1}\in A'} \cdot \EE{ \frac{\ind{i_{t+1}\in A}}{p_{t+1}(A)}\;\middle|\; i_{t} \in A',p_{t+1},d_t } \;\middle|\;d_t,p_{t+1} }
~.
\end{align}
Taking \cref{eq:in,eq:out} together and taking the expectation over $p_{t+1}$, we obtain that for every $d<d_t$:
\begin{align*}
\EE{\frac{\ind{i_{t+1}\in A}}{p_{t+1}(A)}\;\middle|\; d_t}
=
\EE{\frac{\ind{i_{t+1}\in A'}}{p_{t+1}(A')}\;\middle|\;d_t}=1
~,
\end{align*}
where last equality follows from \cref{eq:larged} as $A' \in \cA_{d_t}$.

To conclude, we showed that for all $d$ we have:
\begin{align*}
\EE{\frac{\ind{i_{t+1}\in A}}{p_{t+1}(A)}\;\middle|\; d_t}
=
1
~.
\end{align*}
Taking the expectation over $d_t$, we obtain the desired result.
\end{proof}

\subsection{Bounding the bias and variance}


Next, we turn to bound the variance of the loss estimates $\tell_t$ and the bias of their expectations from the true loss vectors.
These bounds would become useful for controlling the expected regret of the underlying multiplicative updates scheme.

We begin with analyzing the bias of our estimator.
The following lemma shows that our estimates are ``optimistic'', in the sense that they always bound the true losses from below, yet they do not overly underestimate the losses incurred by the algorithm.
The proof is somewhat involved, as a result of the ``bad events'' $B_t$ under which the estimated loss vectors $\tell_t$ are being zeroed-out, thereby introducing biases into the estimation.

\begin{lemma} \label{lem:unbiased}
For all $t$, we have $\E[\tell_t(i)] \le \ell_t(i)$ and $\E[\ell_t(i_t)] \le\E[p_t \cdot \tell_t] + \eta k\log_2{k}$.
\end{lemma}

\begin{proof}
Observe that, by \cref{eq:property} of \cref{lem:sampling},
\begin{align*}
\E[\bell_{t,0}(i)]
=
\ell_t(i) \, \EE{ \frac{\ind{i_t = i}}{p_t(i)} }
=
\ell_t(i)
~.
\end{align*}
We now prove that $\E[\tell_{t,0}(i)] \le \E[\bell_{t,0}(i)]$ for all $i$, which would imply the first claim.
Denote $B_t = \lrset{i \mid p_t(A_d(i)) < 2^d\eta ~ \text{for some} ~ 0 \le d < D}$.
Then, by construction we have
$
\E[\tell_t(i) \mid i_t \in B_t]
=
0
\le
\E[\bell_{t,0}(i) \mid i_t \in B_t]
.
$
Also, since $\E[\sig_{t,d}] = 0$ and $\sig_{t,d}$ is independent of $i_t$ and $\bell_{t,d}$ (the latter only depends on $\sig_{t,0},\ldots,\sig_{t,d-1}$), we have
\begin{align} \label{eq:notBt}
\E[\tell_t \mid i_t \notin B_t]
=
\E[\bell_{t,0} \mid i_t \notin B_t] + \sum_{d=0}^{D-1} \E[\sig_{t,d}] \, \E[\bell_{t,d} \mid i_t \notin B_t]
=
\E[\bell_{t,0} \mid i_t \notin B_t]
~.
\end{align}
Together, we obtain $\E[\tell_{t,0}(i)] \le \E[\bell_{t,0}(i)]$ as required.

Next, to bound $\E[\ell_t(i_t)]$ observe that $\E[p_t \cdot \tell_t \mid i_t \in B_t] = 0$ and, similarly to \cref{eq:notBt},
\begin{align*}
\E[p_t \cdot \tell_t \mid i_t \notin B_t]
=
\E[p_t \cdot \bell_{t,0} \mid i_t \notin B_t]
=
\E[\ell_t(i_t) \mid i_t \notin B_t]
~.
\end{align*}
Denote $\beta_t= \Pr\left[i_t\in B_t\right]$.
Then
\begin{align*}
\E[\ell_t(i_t)]
&=
\beta_t \E[\ell_t(i_t) \mid i_t \in B_t] + (1-\beta_t) \E[\ell_t(i_t) \mid i_t \notin B_t]
\\
&\le
\beta_t + (1-\beta_t) \E[p_t \cdot \tell_t \mid i_t \notin B_t]
\\
&=
\beta_t + \E[p_t \cdot \tell_t]
~,
\end{align*}
where for the inequality we used the fact that $\ell_t(i_t) \le 1$.

To complete the proof, we have to show that $\beta_t \le \eta k\log_2{k}$.
To this end, write
\begin{align*}
\Pr[i_t \in B_t]
\le
\sum_{d=0}^{D-1} \Pr[p_t(A_d(i_t)) < 2^d\eta]
~.
\end{align*}
Using \cref{eq:property} to write
\begin{align*}
\EE{ \frac{1}{p_t(A_d(i_t))} }
=
\sum_{i=1}^k \frac{1}{\abs{A_d(i)}} \EE{ \frac{\ind{i_t \in A_d(i)}}{p_t(A_d(i))} }
=
\sum_{i=1}^k \frac{1}{\abs{A_d(i)}}
=
\abs{\cA_d}
=
\frac{k}{2^d}
\end{align*}
together with Markov's inequality, we obtain
\begin{align*}
\Pr\!\lrbra{ p_t(A_d(i_t)) < 2^d \eta }
=
\Pr\!\lrbra{ \frac{1}{p_t(A_d(i_t))} > \frac{1}{2^d \eta} }
\le
\frac{k}{2^d} \cdot 2^d \eta
=
k\eta
~.
\end{align*}
We conclude that $\beta_t = \Pr[i_t \in B_t] \le \eta k\log_2{k}$, as required.
\end{proof}

Our next step is to bound the relevant variance term of the estimator $\tell_t$.

\begin{lemma} \label{lem:variance}
For all $t$, we have $\E[p_t \cdot \tell_t^2] \le 2k\log_2{k}$.
\end{lemma}

\begin{proof}
Observe that
\begin{align*}
\tell_t^2(i)
\le
\lr{ \bell_{t,0}(i) + \sum_{d=0}^{D-1} \sig_{t,d} \bell_{t,d}(i) }^2
~.
\end{align*}
Since $\E[\sig_{t,d}] = 0$ and $\E[\sig_{t,d} \sig_{t,d'}] = 0$ for all $d \ne d'$, we have for all $i$ that
\begin{align} \label{eq:var1}
\E[\tell_t^2(i)]
=
\E[\tell_{t,0}^2(i)] +
\sum_{d=0}^{D-1} \E[ \bell_{t,d}^2(i) ]
\le
2\sum_{d=0}^{D-1} \E[ \bell_{t,d}^2(i) ]
~.
\end{align}
On the other hand, for all $d$ we have by \cref{lem:bell} that
\begin{align*}
p_t \cdot \bell_{t,d}^2
&\le
\frac{\sum_{i=1}^k p_t(i) \ind{i_t \in A_d(i)}}{p_t(A_d(i_t))^2} \prod_{h=0}^{d-1} (1+\sig_{t,h})^2
\\
&=
\frac{1}{p_t(A_d(i_t))} \prod_{h=0}^{d-1} (1+\sig_{t,h})^2
\\
&=
\sum_{i=1}^k \frac{1}{\abs{A_d(i)}} \frac{\ind{i_t \in A_d(i)}}{p_t(A_d(i))} \prod_{h=0}^{d-1} (1+\sig_{t,h})^2
~.
\end{align*}
Since $i_t$ is independent of the $\sig_{t,h}$, and recalling \cref{eq:property}, we get
\begin{align*}
\E_t[p_t \cdot \bell_{t,d}^2]
\le
\sum_{i=1}^k \frac{1}{\abs{A_d(i)}} \EE{ \frac{\ind{i_t \in A_d(i)}}{p_t(A_d(i))} } \prod_{h=0}^{d-1} \E[(1+\sig_{t,h})^2]
=
\sum_{i=1}^k \frac{2^d}{\abs{A_d(i)}}
=
2^d \abs{\cA_d}
=
k
~.
\end{align*}
Together with \cref{eq:var1}, this gives
\begin{align*}
\E[p_t \cdot \tell_t^2]
\le
2\sum_{d=0}^{D-1} \E[p_t \cdot \bell_{t,d}^2]
\le
2k\log_2{k}
&.\qedhere
\end{align*}
\end{proof}

\subsection{Concluding the proof}

To conclude the proof and obtain a regret bound, we will use the following well-known second-order regret bound for the multiplicative weights (MW) method, essentially due to \cite{cesa2007improved} (see also \cite{alon2015online} for the version given here).
For completeness, we give a proof of this bound in \cref{sec:mabproofs} below.

\begin{lemma}[Second-order regret bound for MW] \label{lem:mw2}
Let $\eta > 0$ and let $c_1,\ldots,c_T \in \reals^k$ be real vectors such that $c_t(i) \ge -1/\eta$ for all $t$ and $i$.
Consider a sequence of probability vectors $q_1,\ldots,q_T \in \Delta_k$ defined by $q_1 = (\tfrac{1}{k},\ldots,\tfrac{1}{k})$, and for all $t > 1$:
\begin{align*}
q_{t+1}(i) = \frac{ q_t(i) \, e^{-\eta c_t(i)} }{ \sum_{j=1}^k
q_t(j) \, e^{-\eta c_t(j)} } \qquad \forall ~ i \in [k] ~.
\end{align*}
Then, for all $i^* \in [k]$ we have that
\begin{align*}
\sum_{t=1}^T q_t \cdot c_t - \sum_{t=1}^T c_t(i^*)
\le
\frac{\log{k}}{\eta} + \eta \sum_{t=1}^T q_t \cdot c_t^2
~.
\end{align*}
\end{lemma}

We now have all we need in order to prove our main result.


\begin{proof}[Proof of \cref{thm:main}]
First, we bound the expected movement cost.
\cref{lem:sampling} says that with probability at least $1-2^{-(d+1)}$, the actions $i_t$ and $i_{t-1}$ belong to the same subtree on level $d$ of the tree, which means that $\Delta(i_t,i_{t-1}) \le 2^d/k$ with the same probability.
Hence,
\begin{align*}
\E[\Delta(i_t,i_{t-1})]
\le
\sum_{d=0}^{D-1} \frac{2^d}{k} \Pr\lrbra{\Delta(i_t,i_{t-1}) > \frac{2^d}{k}}
\le
\sum_{d=0}^{D-1} \frac{1}{2k}
=
\frac{\log_2{k}}{2k}
~,
\end{align*}
and the cumulative movement cost is then $\O((T/k)\log{k})$.

We turn to analyze the cumulative loss of the algorithm.
We begin by observing that $\tell_t(i) \ge -1/\eta$ for all $t$ and $i$.
To see this, notice that $\tell_t = 0$ unless $i_t \notin B_t$, in which case we have, by \cref{lem:bell} and the definition of $B_t$,
\begin{align*}
0
\le
\bell_{t,d}(i)
\le
\frac{2^d}{p_t(A_d(i_t))}
\le
\frac{1}{\eta}
\qquad\quad
\forall ~ 0 \le d < D
~,
\end{align*}
and since $\tell_t$ has the form $\tell_t = \bell_{t,0} + \sum_{h=0}^{d_t-1} \bell_{t,h} - \bell_{t,d_t}$ (recall \cref{eq:tell-equiv}), we see that $\tell_t(i) \ge -1/\eta$.
Hence, we can use second-order bound of \cref{lem:mw2} on the vectors $\tell_t$ to obtain
\begin{align*} 
\sum_{t=1}^T p_t \cdot \tell_t - \sum_{t=1}^T \tell_t(i^*)
\le
\frac{\log k}{\eta} + \eta \sum_{t=1}^T p_t \cdot \tell_t^2
\end{align*}
for any fixed $i^* \in [k]$.
Taking expectations and using \cref{lem:unbiased,lem:variance}, we have
\begin{align*}
\EE{ \sum_{t=1}^T \ell_t(i_t) } - \sum_{t=1}^T \ell_t(i^*)
\le
\frac{\log_2 k}{\eta} + 2\eta T k\log_2{k}
~.
\end{align*}
Choosing $\eta = 1/\sqrt{Tk}$, we get a regret bound of $\O(\sqrt{Tk}\log{k})$.
\end{proof}

\subsection{Additional technical proofs}
\label{sec:mabproofs}

Here we give a proof of our technical lemma bounding the magnitude of the balancing terms $\bell_{t,d}$.

\begin{proof}[Proof of \cref{lem:bell}]
We will prove the claim by induction on $d$.
For the base case $d=0$, \cref{eq:bell1} follows directly from our definitions and the fact that $0 \le \ell_t(i) \le 1$ for all $i$.
Next, we prove that \cref{eq:bell1} holds for some $d$ assuming it hold for all $d' < d$.
Since $(1+\sig_{t,d-1})\bell_{t,d-1}(i) \ge 0$ for all $i$ by the induction hypothesis, the recursive definition of $\bell_{t,d}$ implies that
\begin{align*}
\bell_{t,d}(i)
\ge
-\frac{1}{\eta} \log\lrBigg{ \sum_{j \in A_{d}(i)} \frac{p_t(j)}{p_t(A_{d}(j))} }
=
0
~.
\end{align*}
Furthermore, the definition of $\bell_{t,d}$ together with the convexity of $-\log{x}$ and Jensen's inequality give
\begin{align*}
\bell_{t,d}(i)
&\le
(1+\sigma_{d-1})\sum_{j \in A_{d}(i)} \frac{p_t(j)}{p_t(A_{d}(j))} \bell_{t,d-1}(j)
\\
&\le
\frac{\ind{i_t \in A_d(i)}}{p_t(A_{d}(i))} \sum_{j \in A_{d-1}(i)} \frac{p_t(j)}{p_t(A_{d-1}(j))} \prod_{h=0}^{d-1} (1+\sig_{t,h})
\\
&=
\frac{\ind{i_t \in A_d(i)}}{p_t(A_{d}(i))} \prod_{h=1}^{d-1} (1+\sig_{t,h})
~,
\end{align*}
where in the second inequality we used the induction hypothesis.
This concludes the inductive argument.
\end{proof}

Finally, for completeness, we give a proof of \cref{lem:mw2} being  central to our regret analysis.

\begin{proof}[Proof of \cref{lem:mw2}]
The proof follows the standard analysis of exponential weighting schemes: let $w_t(i) = \exp\lrbig{\!-\eta\sum_{s=1}^{t-1} c_s(i)}$ and let $W_t = \sum_{i \in V} w_t(i)$. Then $q_t(i) = w_t(i)/W_t$ and we can write
\begin{align*}
\frac{W_{t+1}}{W_t}
&= \sum_{i=1}^k \frac{w_{t+1}(i)}{W_t}\\
&= \sum_{i=1}^k \frac{w_{t}(i)\,\exp\bigl(-\eta\,c_{t}(i)\bigr)}{W_t}\\
&= \sum_{i=1}^k q_{t}(i)\,\exp\bigl(-\eta\,c_{t}(i)\bigr)\\
&\le \sum_{i=1}^k q_{t}(i)\,\left(1 - \eta c_{t}(i) + \eta^2 c_{t}(i)^2\right)
\\
&= 1 - \eta\,\sum_{i=1}^k q_{t}(i) c_{t}(i) + \eta^2\,\sum_{i=1}^k q_{t}(i) c_{t}(i)^2~,
\end{align*}
where the inequality uses the inequality $e^{x} \le 1+x+x^2$ valid for $x \le 1$.
Taking logarithms, using $\log(1-x) \le -x$ for all $x \le 1$, and summing over $t =
1, \ldots, T$ yields
\[
\log\frac{W_{T+1}}{W_1} 
\le 
\sum_{t=1}^T \sum_{i=1}^k \lr{ -\eta\, q_{t}(i) c_{t}(i) +\eta^2\,q_{t}(i) c_{t}(i)^2 }
~.
\]
Moreover, for any fixed action $i^*$, we also have
\[
\log \frac{W_{T+1}}{W_1} 
\ge 
\log \frac{w_{T+1}(k)}{W_1} 
= 
-\eta\,\sum_{t=1}^T c_{t}(i^*) - \log{k}
~.
\]
Putting together and rearranging gives the result.
\end{proof}

\subsection{Learning Continuum--Arm Bandit with Lipschitz Loss Functions}\label{sec:lipschitz}

In this section we turn to show how to reduce the problem of learning Lipschitz functions to MAB with tree-metric movement costs. Specifically we aim at proving \thmref{thm:lipschitz}. Specifically we prove the following statement,
\begin{theorem}\label{thm:lipschitz2}
Set $k=L^{2/3}T^{1/3}$ and $\eta=1/\sqrt{kT}$. Consider a procedure that receives actions from \cref{alg:alg1} and returns as feedback $f_t(\frac{i_t}{k})$ then for every sequence of
$L$-Lipschitz loss functions $f_1,\ldots,f_T$ and an $L$-Lipschitz metric $\distance$, we have that:
\begin{align*}
\mregret(f_{1:T}, \distance) 
= 
\tO\lrbig{ L^{1/3}T^{2/3} } 
~.
\end{align*}
In particular, the result holds for $L\ge 1$ and $\distance(x_t,x_{t+1})=|x_t-x_{t+1}|$.
\end{theorem}

\begin{proof}
First note that for every $x^*\in [0,1]$ we can find $x = \{\frac{1}{k},\frac{2}{k},\ldots,1\}$ such that $f_t(x)-f_t(x^*) \le L/k = L^{1/3} T^{-1/3}$, hence
\begin{align*}
\sum_{t=1}^T \lrbig{ f_t(x)-f_t(x^*) }
= 
L^{1/3}T^{2/3}.
\end{align*}
Therefore if we can show that the regret against every $x^*\in \{\frac{1}{k},\frac{2}{k},\ldots 1\}$  is bounded by $O(L^{1/3}T^{2/3})$ we obtain that the same regret bound is true for every $x\in [0,1]$.

Next, we apply \cref{alg:alg1} on the a fully balanced tree where we associate with the leaves $\{1,\ldots, k\}$ the actions $\{\frac{1}{k},\frac{2}{k} \ldots,1\}$. One can then show that $ \frac{|i-j|}{k} \le \tdistance(i,j)$.
We then obtain by \thmref{thm:main} that for every $x\in \{\frac{1}{k},\frac{2}{k} \ldots,1\}$:
\begin{align*}
\EE{\sum_{t=1}^T f_t(x_t)} - \min_{x} \sum_{t=1}^T f_t(x) = O(\eta k T) = \tO(L^{1/3}T^{2/3})
~.
\end{align*}
As to the second term in the regret we obtain that
\begin{align*}
\EE{ \sum_{t=1}^T \distance(x_t,x_{t+1})} \le {L \sum_{t=1}^T|x_t-x_{t+1}|} \le \EE{ L\sum_{t=1}^T \tdistance(i_t,i_{t+1})}
= \tO\left(L \frac{T}{k}\right) 
= \tO(L^{1/3}T^{2/3})
~.
\end{align*}
Taken together we obtain that
\begin{align*}
\EE{\sum_{t=1}^T f_t(x_t) + \sum_{t=1}^T \distance(x_t,x_{t+1})} -\min_{x\in \{\frac{1}{k},\ldots, 1\}} \sum_{t=1}^T f_t(x)  
= \tO(L^{1/3}T^{2/3})
~.&\qedhere
\end{align*}
\end{proof}


\section{Online Pricing with Patient Buyers}
\label{sec:pricing}


In this section we present our reduction of adaptive pricing with patient buyers to a MAB with movement costs.

The reduction is presented in \cref{alg:pp} and uses our algorithm for MAB with movement costs (\cref{alg:alg1}) as a black-box.
The algorithm divides the time interval $T$ into $\taumax$ blocks and the updates the price on $\TT=T/\taumax$ rounds.
At each round $t$ the algorithm publishes a fixed price for the whole block of $\taumax$ consecutive days. Then, as feedback, the algorithm receives the mean revenue for those days, which we denote by
$$
r'_t
=
\frac{1}{\taumax} \sum_{k=(t-1)\taumax+1}^{t\taumax} \buy{\bb_k}{\p_k,\ldots,\p_{k+\taumax}}
~.
$$
Thus, we can consider the algorithm as an online algorithm over $\TT$ rounds: where at each round $t$ the algorithm announces a fixed action $\p'_{t+1}$ (the price for the next $\taumax$ days) and receives at the end of the round as feedback $r'_t$.
Note that prices are always announced $\taumax$ days in advance, as required.

The algorithm draws $\beta_1,\ldots, \beta_{\TT}$ unbiased Bernoulli random variables, and this sequence determines the switches in prices and updates. The algorithm posts a new price only on rounds where $\beta_{t}=0$ and $\beta_{t+1}=1$, and invoke the update rule of \cref{alg:alg1} only on rounds where $\beta_{t+1}=0$ and $\beta_{t+2}=1$. Note that these two events never co-occur, and further the algorithm exploits the feedback only on days prior to a switch, thus guaranteeing that the feedback is always on days when prices are fixed throughout the present and future block.

\begin{myalgorithm}[ht]
\wrapalgo[0.75\textwidth]{
\textbf{Parameters:} horizon $T$, and maximal patience $\taumax$\\
Initialize, $\TT=T/(2\taumax)$, $k={\TT}^{1/3}$, $\eta=2/\sqrt{\TT k}$\\
Initialize an instance $B$ of $\klm(k,\eta)$\\
Draw i.i.d.~unbiased Bernoulli r.v. $\beta_0,\ldots, \beta_{\TT}$\\
Sample $i_1\sim B$, set $\p'_1=i_1/k$\\
Announce prices $\p_1=\p_2=\ldots,p_{\taumax}=\p_1'$\\
For $t=1,\ldots, T$\;
\begin{enumerate}[nosep,label=(\arabic*)]
\item If $\beta_t=0$ and $\beta_{t+1}=1$, sample $i_{t+1}\sim B$; otherwise set $i_{t+1}=i_t$
\item Set $\p'_{\TT+1}=i_{t+1}/k$ and announce prices:
$\p_{t\taumax+1}= \cdots = \p_{(t+1)\taumax}= \p'_{t+1}$
\item Collect revenues $r_{(t-1)\taumax+1},\ldots,r_{t\taumax}$
and set $$r'_{t}(\p'_t) = \frac{1}{\taumax} \sum_{k=(t-1)\taumax+1}^{t\taumax} r_{k}$$
\item If $\beta_{t+1}=0, \beta_{t+2}=1$, update $B$ with feedback $f_t=1-r'_t(\p'_t)$
\end{enumerate}
}
\caption{Adaptive pricing with patient buyers.} \label{alg:pp}
\end{myalgorithm}

As discussed briefly in \cref{sec:techniques}, the main difficulty in reducing the adaptive pricing problem to MAB, which \cref{alg:pp} overcomes, is in that the feedback function is not only a function of the current posted price (which is in fact the price tomorrow) but also of past prices.
For example, for $\taumax=1$ the revenue at time $t$ is a function of $\p_t$ and $\p_{t+1}$, where only $\p_{t+1}$ needs be posted at time~$t$.
\cref{alg:pp} overcomes this issue by employing techniques from \cite{dekel2014blinded} for handling adaptive feedback.
The tools developed there allow regret minimization when feedback is taken only in time steps when the price is fixed for a period of time.
Relying on these techniques,
we construct an algorithm that produces a sequence of prices with low regret \emph{if} each buyer $\bb_t$ would observe price $\p_t$.
However, in our setting, a buyer may buy at a consecutive time steps; 
the additional cost we suffer is bounded by the potential cost of switching to lower prices, namely, by the movement cost of the algorithm.


The main result of this section, stated earlier in \cref{thm:pp}, shows that \cref{alg:pp} attains a regret bound of $O(\taumax^{1/3}T^{2/3})$ against any sequence of buyers with patience at most $\taumax$:



The remainder of the section is devoted to proving \cref{thm:pp}.
We begin by establishing additional notation required for the proof.
We will denote the expected revenue from the buyers at each block as follows:
\[
\bbb_t(\p'_t,\p'_{t+1}) =  \frac{1}{\taumax}\sum_{k=t\taumax+1}^{(t+1)\taumax} \buy{\bb_k}{\p_k,\ldots,\p_{k+\tau_t}}
~.
\]
Note that since the blocks are of size $\taumax$, each buyer can see at most prices that are published on the next block, hence $\p_{k+\tau_t}$ either equals $\p'_t$ or $\p'_{t+1}$.
In turn, this means that the expected revenue is indeed a function of $\p'_t$ and $\p'_{t+1}$ alone.

We will further denote the expected revenue from buyers if they observe only the price at time of arrival as follows:
\[
\bbb_t(\p'_t) =  \frac{1}{\taumax}\sum_{k=t\taumax+1}^{(t+1)\taumax} \buy{\bb_k}{\p'_t,\ldots,\p'_t}
~.
\]

First, we are estimating the performance on the subsequence of rounds where the algorithm exploits the received feedback.

\begin{lemma}\label{lem:Sbound}
Let $\beta_1,\ldots,\beta_T$ be a sequence of unbiased Bernoulli random variables, denote
\[S= \{ t\in [\TT]~:~ \beta_{t+1}=0, \beta_{t+2}=1\},\]
and denote the elements of $S$ in increasing order  $S=\{t_{s_1} \le t_{s_2},\ldots, \le t_{s_{|S|}}\}$.
For any price $\p^*\in \{\frac{1}{k},\frac{2}{k},\ldots, 1\}$, \cref{alg:pp} enjoys the following guarantee:
\[\EE{ \sum_{t\in S}  \bbb_t(\p^*) -\bbb_t(\p'_t)} = \tO(\TT^{2/3})~,\]
and
\[\EE{\sum_{s=1}^{|S|} |\p'_{t_s}-\p'_{t_{s+1}}|} = \tO(\TT^{2/3})~.\]
\end{lemma}

\begin{proof}
For each sequence of buyers $\bb_1,\ldots,\bb_{T}$, define a sequence of loss functions $\ell_1\ldots, \ell_{\TT}$ according to:
\[
\ell_t(i) = 1-\bbb_t\lr{\frac{i}{k}}
~.
\]
First note that for every $t\in S$ we have $\p'_t=\p'_{t+1}$. The algorithm, in turn, announces the same price $\p'_t$ for all days: $\{(t-1)\taumax +1,\ldots,{(t+1)\taumax}\}$, hence the revenue obtained from buyer $\bb_k$ for every $(t-1)\taumax +1 \le k \le t\taumax$ is given by $\bb_t(\p'_t,\p'_t)$.
Hence, the feedback used to update the algorithm $B$ at round $t$ is
\begin{align*}
f_t=1-r'_t=1- \frac{1}{\taumax} \sum_{k=(t-1)\taumax+1}^{t\taumax} \buy{\bb_k}{\p_k,\ldots, \p_{k+\taumax}} = 1- \sum_{k=(t-1)\taumax+1}^{t\taumax}\frac{1}{\taumax}\buy{\bb_{k}}{\p'_t}=\ell_t(i_t)
~.
\end{align*}
In words, we have shown that at every step $t\in S$, \cref{alg:pp} receive action $i_t$ and return to \cref{alg:alg1} as feedback $\ell_t(i_t)$.
Thus \cref{alg:pp} applies \cref{alg:alg1} on the sequence of losses $\{\ell_t\}_{t\in S}$.
As a corollary we have that:
\[
\EE{ \sum_{t\in S}  \bbb_t(\p^*) -\bbb_t(\rho'_t)\;\middle|\;S}=\EE{ \sum_{t\in S}  \ell_t(i^*) -\ell_t(i_t)\;\middle|\; S} = O(\eta k |S|)
~.
\]
Taking expectation over $S$ and noting $\EE{\abs{S}} = \frac{1}{4} \TT$ we get that
\[
\EE{ \sum_{t\in S} \bbb_t(\p^*) - \bbb_t(\p'_t,) }
=
O(\TT^{2/3})
~.
\]

As in \secref{sec:lipschitz}, note that if we associate with the prices the corresponding actions on the tree we obtain that $|\p'_t-\p'_{t+1}| \le \tdistance(i_t,i_{t+1})$ hence we obtain as a second guarantee that the movement cost of the algorithm is given by
\begin{align*}
\EE{ \sum_{s=1}^{|S|} |\p'_{t_s}-\p'_{t_{s-1}}|\;\middle|\; S}
=
\EE{ \sum_{s=1}^{|S|} \tfrac{1}{k} |i_{t_s}-i_{t_{s-1}}| \;\middle|\; S}
\le
\EE{ \sum_{s=1}^{|S|} \tfrac{1}{k} \Delta(i_{t_s},i_{t_{s-1}}) \;\middle|\; S}
=
O\left(\tfrac{1}{k}|S|\right)
~.
\end{align*}
Again taking expectation over $S$ we get that
\[
\EE{ \sum_{s=1}^{|S|} |\p'_{t_s}-\p'_{t_{s-1}}| }
=
\tO\lr{ \tfrac{1}{k} \TT }
~.\qedhere
\]
\end{proof}

Next, we upper bound the regret over the expected regret over the blocks of buyers, $\bar{\bb}_t$:

\begin{lemma}\label{lem:buyerregret}
For every $\p^*\in \{\frac{1}{k},\frac{2}{k},\ldots, 1\}$ we have that
\begin{align*}
\EE{\sum_{t=1}^{\TT} \bbb_t(\p^*) -\buy{\bbb_t}{\p'_t,\p'_{t+1}}}\le 4\EE{ \sum_{t\in S} \bbb_t(\p^*) -\bbb_t(\rho'_t)} +  \EE{\sum_{s=1}^{|S|}|\p'_{t_s}-\p'_{t_{s-1}}|}
~.
\end{align*}
\end{lemma}

\begin{proof}
First note that for every $\p^*$ we have
\[
\EE{\sum_{t\in S}  \bbb_t(\p^*) }
=
\EE{\sum_{t=1}^T  \bbb_t(\p^*)\beta_{t+2}(1-\beta_{t+1})}
~.
\]
Since the Bernoulli random variables are independent of $\bb_t$ and $\p^*$ we get that
\begin{align}\label{eq:pstar}
\EE{\sum_{t\in S} \buy{\bbb_t}{\p^*}}
=
\EE{\sum_{t=1}^T  \bbb_t(\p^*)\beta_{t+2}(1-\beta_{t+1})}
=
\frac{1}{4}\EE{\sum_{t=1}^T  \bbb_t(\p^*)}
~.
\end{align}
Similarly we have that
\[\EE{\sum_{t\in S} \bbb_t(\p'_t)}=\EE{\sum_{t=1}^{\TT} \bbb_t(\p'_t)\beta_{t+2}(1-\beta_{t+1})} =\frac{1}{4} \EE{\sum_{t=1}^{\TT} \bbb_t(\p'_t)}~, \]
where the equality holds since $\p'_{t}$ is independent of $\beta_{t+1}$ and $\beta_{t+2}$.
We can bound $\buy{\bb_t}{\p'_t,\p'_{t+1}} \ge \buy{\bb_t}{\p'_t,\p'_t} - |\p'_t-\p'_{t+1}|$.
Hence
$
\bbb_t(\p'_t,\p'_{t+1}) \ge \buy{\bbb}{{\p'_t}} - |\p'_t-\p'_{t+1}|
,
$
and we obtain:
\begin{align} \label{eq:pt}
\EE{\sum_{t=1}^{\TT} \bbb_t(\p'_t,\p'_{t+1})}
&\ge
\EE{\sum_{t=1}^{\TT} \bbb_t(\p'_t)- |\p'_t-\p'_{t+1}|}
\notag\\
&=
4 \EE{\sum_{t\in S} \bbb_t(\p'_t)} - \sum_{t=1}^{\TT} \EE{|\p'_t-\p'_{t+1}|}
\notag\\
&=
4\EE{\sum_{t\in S} \bbb_t(\p'_t)} -\EE{\sum_{t=s}^{|S|}|\p'_{t_s}-\p'_{t_{s-1}}|}
~,
\end{align}
where last equality is true since, we have that $\p'_t=\p'_{t+1}$ unless $\p'_{t-1}\in S$ in which case we have that $\p'_{t-1}=\p'_{t} = \p'_{t_{s}}$ for some $s$ and $\p'_{t+1}= \p'_{t_{s+1}}$. Taken together with \cref{eq:pstar,eq:pt} we obtain the desired result.
\end{proof}

We are now ready to prove the main result of this section.


\begin{proof}[Proof of \thmref{thm:pp}]

First, for any $\p\in \{\frac{1}{k},\ldots, 1\}$, by employing \cref{lem:buyerregret} we have the following:
\begin{align*}
\EE{\sum_{t=1}^T  \buy{\bb_t}{\p,\ldots,\p} - \buy{\bb_t}{\p_t,\ldots,\p_{t+\taumax}}}
&=
\sum_{t'=1}^{\TT} \sum_{t=(t'-1)\taumax+1}^{t'\taumax} \lrbig{\buy{\bb_t'}{\p,\ldots,\p} - \buy{\bb_t}{\p_t,\ldots,\p_{t+\taumax}}}
\\
&=
\taumax\EE{\sum_{t=1}^{\TT} \buy{\bbb_t}{\p} -  \buy{\bbb_t}{\p'_t,\p'_{t+1}}}
\\
&\le
\frac{\taumax}{4}\EE{ \sum_{t\in S} \bbb_t(\p) -\bbb_t(\rho'_t)} + \taumax \EE{\sum_{s=1}^{|S|}|\p'_{t_s}-\p'_{t_{s-1}}|}
~.
\end{align*}
Next, for any $\p^* \in [0,1]$ there exist $\p \in \{\frac{1}{k},\ldots,1\}$ such that $\p^*>\p$ and
$\buy{\bb_t}{\p^*,\ldots, \p^*} < \buy{\bb_t}{\p,\ldots,\p}+ \frac{1}{k}$.
Hence, for every $\p^* \in [0,1]$ we obtain that
\begin{align*}
\sum_{t=1}^T \buy{\bb_t}{\p^*,\ldots ,\p^*}
&- \EE{\sum_{t=1}^T \buy{\bb_t}{\p_{t},\ldots \p_{t+\taumax}}}
\\
&\le
\frac{\taumax}{4}\EE{ \sum_{t\in S} \bbb_t(\p) -\bbb_t(\p'_t)} + \taumax \EE{\sum_{s=1}^{|S|}|\p'_{t_s}-\p'_{t_{s-1}}|}+O(\tfrac{T}{k})
~.
\end{align*}
By \cref{lem:Sbound} we now obtain
\begin{align*}
\sum_{t=1}^T \buy{\bb_t}{\p^*,\ldots ,\p^*}
- \EE{\sum_{t=1}^T \buy{\bb_t}{\p_t,\ldots,\p_{t+\taumax}}}
=
O\lr{ \sqrt{\taumax k\TT} + \frac{\taumax \TT}{k} + \frac{T}{k} }
= O(\taumax^{1/3} T^{2/3})
~,
\end{align*}
and using our choice of $k$ gives the result.
\end{proof}

\bibliographystyle{plainnat}

\bibliography{paper}

\end{document}